\def\year{2020}\relax
\documentclass[letterpaper]{article} 
\usepackage{aaai20}  
\usepackage{times}  
\usepackage{helvet} 
\usepackage{courier}  
\usepackage[hyphens]{url}  
\usepackage{graphicx} 
\usepackage{graphicx}  
\usepackage[protrusion=true,expansion=true]{microtype}		

\usepackage{amsmath}
\usepackage{amsthm}
\usepackage{dsfont}
\usepackage{amssymb}
\usepackage{mathtools}
\usepackage{mathrsfs}
\mathtoolsset{showonlyrefs}

\newtheorem{thm}{Theorem}[]

\newtheorem{prop}{Property}[]

\newtheorem{cor}{Corollary}[]

\usepackage{courier} 

\usepackage{lipsum} 

\usepackage{subcaption}

\usepackage{xcolor}

\usepackage{algpseudocode}
\usepackage[vlined,linesnumbered,ruled,algo2e]{algorithm2e}
\SetKwProg{Fn}{Function}{}{}
\let\oldnl\nl
\newcommand{\nonl}{\renewcommand{\nl}{\let\nl\oldnl}}

\newcommand{\citet}[1]{\citeauthor{#1} \shortcite{#1}}
\newcommand{\citep}{\cite}

\usepackage{tikz}
\usetikzlibrary{decorations.pathreplacing,calc}
\newcommand{\tikzmark}[1]{\tikz[overlay,remember picture] \node (#1) {};}

\newcommand*{\AddNote}[4]{%
    \begin{tikzpicture}[overlay, remember picture]
        \draw [decoration={brace,amplitude=0.5em},decorate,ultra thick,blue]
            ($(#3)!(#1.north)!($(#3)-(0,1)$)$) --  
            ($(#3)!(#2.south)!($(#3)-(0,1)$)$)
                node [align=center, text width=2.5cm, pos=0.5, anchor=west] {#4};
    \end{tikzpicture}
}%

\title{Lifelong Learning with a Changing Action Set}
\author{
Yash Chandak\textsuperscript{\rm 1} \,\,\, Georgios Theocharous\textsuperscript{\rm 2} \,\,\, Chris Nota\textsuperscript{\rm 1} \,\,\, Philip S. Thomas\textsuperscript{\rm 1} \\
\textsuperscript{\rm 1}University of Massachusetts Amherst, \textsuperscript{\rm 2}Adobe Research \\
\{ychandak,cnota,pthomas\}@cs.umass.edu \,\,\,\, theochar@adobe.com
}
 
\begin{document}

\maketitle

\begin{abstract}
In many real-world sequential decision making problems, the number of available actions (decisions) can vary over time. 
%
%
While problems like catastrophic forgetting, changing transition dynamics, changing rewards functions, etc.~have been well-studied 
in the lifelong learning literature, the setting where the size of the action set changes 
remains 
unaddressed.
In this paper, we present first steps towards developing an algorithm that autonomously adapts to an action set whose size changes over time. 
To tackle this open problem, we break it into two problems that can be solved iteratively: inferring the underlying, unknown, structure in the space of actions and optimizing a policy that leverages this structure.  
%
%
We demonstrate the efficiency of this approach on large-scale real-world lifelong learning problems.

\end{abstract}

\section{Introduction}

Real-world problems are often non-stationary. That is, parts of the problem specification change over time. 
We desire autonomous systems that continually adapt by capturing the regularities in such changes, without the need to learn from scratch after every change. 
In this work, we address one form of \textit{lifelong learning} for sequential decision making problems, wherein the set of possible actions (decisions) varies over time. 
Such a situation is omnipresent in real-world problems.
For example, in robotics it is natural to add control components over the lifetime of a robot to enhance its ability to interact with the environment. 
In hierarchical reinforcement learning, an agent can create new 
\textit{options} \citep{between_MDPs} over its lifetime, which are in essence new actions. 
In medical decision support systems for drug prescription, new procedures and medications are continually discovered. 
%
%
%
%
In product recommender systems, new products are constantly added to the stock,  
and in tutorial recommendation systems, new tutorials are regularly developed, thereby continuously increasing the number of available actions for a recommender engine.
%
%
%
%
%
These examples capture the broad idea that, for an agent that is deployed in real world settings, the possible decisions it can make changes over time, and motivates the question that we aim to answer: 
\textit{how do we develop algorithms that can continually adapt to such changes in the action set over the agent's lifetime?}   

\emph{Reinforcement learning} (RL) has emerged as a successful class of methods for solving sequential decision making problems. 
However, excluding notable exceptions that we discuss later \citep{Boutilier2018PlanningAL,mandel2017add}, its applications have been limited to settings where the set of actions is fixed. 
This is likely because RL algorithms are designed to solve a mathematical formalization of decision problems called \emph{Markov decision processes} (MDPs) \citep{puterman2014markov}, wherein the set of available actions is fixed. 
To begin addressing our lifelong learning problem, we first extend the standard MDP formulation to incorporate this aspect of changing action set size. 
Motivated by the regularities in real-world problems, we consider an underlying, unknown, structure in the space of actions from which new actions are generated. 
We then theoretically analyze the difference between what an algorithm can achieve with only the actions that are available at one point in time, and the best that the algorithm could achieve if it had access to the entire underlying space of actions (and knew the structure of this space). 
%
Leveraging insights from this theoretical analysis, we then study how the structure of the underlying action space 
can be recovered from interactions 
with the environment, and how algorithms can be developed to use this structure to facilitate lifelong learning. 

As in the standard RL setting, when facing a changing action set, 
the parameterization of the policy plays an important role.
The key consideration here is how to parameterize the policy and adapt its parameters when the set of available actions changes. 
%
To address this problem, we leverage the structure in the underlying action space to parameterize the policy such that it is invariant to the cardinality of the action set---changing the number of available actions does not require changes to the number of parameters or the structure of the policy. 
Leveraging the structure of the underlying action space also improves generalization by allowing the agent to infer the outcomes of actions similar to actions already taken.
These advantages make our approach ideal for lifelong learning problems where the action set changes over time, and where quick adaptation to these changes, via generalization of prior knowledge about the impact of actions, is beneficial.  
%
%

\section{Related Works}
Lifelong learning is a well studied problem 
\citep{thrun1998lifelong,ruvolo2013ella,silver2013lifelong,chen2016lifelong}.
Predominantly, prior methods aim to address catastrophic forgetting problems in order to leverage prior knowledge for new tasks \citep{french1999catastrophic,kirkpatrick2017overcoming,lopez2017gradient,zenke2017continual}.
Several meta-reinforcement-learning methods aim at addressing the problem of transfer learning, few-shot shot adaption to new tasks after training on a distribution of similar tasks, and automated hyper-parameter tuning \citep{xu2018meta,gupta2018meta,wang2017learning,duan2016rl,finn2017model}.
Alternatively, many lifelong RL methods consider learning online in the presence of \textit{continuously} changing transition dynamics or reward functions \citep{neu2013online,gajane2018sliding}. 
In our work, we look at a complementary aspect of the lifelong learning problem, wherein the size of the action set available to the agent change over its lifetime. 

%

%
%
Our work also draws inspiration from recent works which leverage action embeddings \citep{dulac2015deep,he2015deep,bajpai2018transfer,chandak2019action,naturalGuy}.
Building upon their ideas, we present a new objective for learning structure in the action space, and show that  
the performance of the policy resulting from using this inferred structure has bounded sub-optimality.  
Moreover, in contrast to their setup 
where the size of the action set is fixed, we consider the case of lifelong MDP, where the number of actions changes over time.

\citet{mandel2017add} and \citet{Boutilier2018PlanningAL} present the work most similar to ours.  
\citet{mandel2017add} consider the setting where humans can provide 
new actions to an RL system. 
%
The goal in their setup is to minimize human effort by querying for new actions only at states 
where new actions are most likely to boost 
performance.
In comparison, our setup considers the case where the  new actions become available through some external, unknown, process 
and the goal is to build learning algorithms that can efficiently adapt to such changes in the action set.
\citet{Boutilier2018PlanningAL} laid the foundation for the stochastic action set MDP (SAS-MDP) setting where there is a fixed, finite, number of (base) actions and the available set of actions is a stochastically chosen subset of this base set.  
While SAS-MDPs can also be considered to have a `changing action set', unlike their work there is no fixed maximum number for the available actions in our framework.
Further, in their setup, there is a \textit{possibility} that within a single long episode an agent can observe \textit{all} possible actions it will ever encounter. 
In our set-up, this is never possible.
As shown by \citet{Boutilier2018PlanningAL}, SAS-MDPs can also be reduced to standard MDPs by extending the state space to include the set of available action.
This cannot be done in our lifelong-MDP setup, as that would imply that the state-space is changing across episodes or the MDP is non-stationary.
The works by \citet{gabel2008reinforcement} and \citet{ferreira2017answer} also consider subsets of the base actions for DEC-MDPs and answer-set programming, respectively,  but all the mentioned differences from the work by \citet{Boutilier2018PlanningAL} are also applicable here. 
%
%

These differences let the proposed work better capture the challenges of lifelong learning, where the cardinality of the action set itself varies over time and an agent has to deal with actions that it has never dealt with before.

\section{Lifelong Markov Decision Process}
\label{sec:lmdp}
        MDPs, the standard formalization of decision making problems, are not flexible enough to encompass lifelong learning problems wherein the action set size changes over time. 
    	In this section we extend the standard MDP framework to model this setting. 

        In real-world problems where the set of possible actions changes, there is often underlying structure in the set of all possible actions (those that are available, and those that may become available). 
        %
        %
        For example, tutorial videos can be described by feature vectors that encode their topic, difficulty, length, and other attributes;
        in robot control tasks, primitive locomotion actions like left, right, up, and down could be encoded by their change to the Cartesian coordinates of the robot, etc. 
        Critically, we will not assume that the agent knows this structure, merely that it exists. 
        If actions are viewed from this perspective, then the set of all possible actions (those that are available at one point in time, and those that might become available at any time in the future) can be viewed as a vector-space, $\mathcal E \subseteq \mathbb R^d$. 

\begin{figure}[t]
		\centering
		\includegraphics[width=0.45\textwidth]{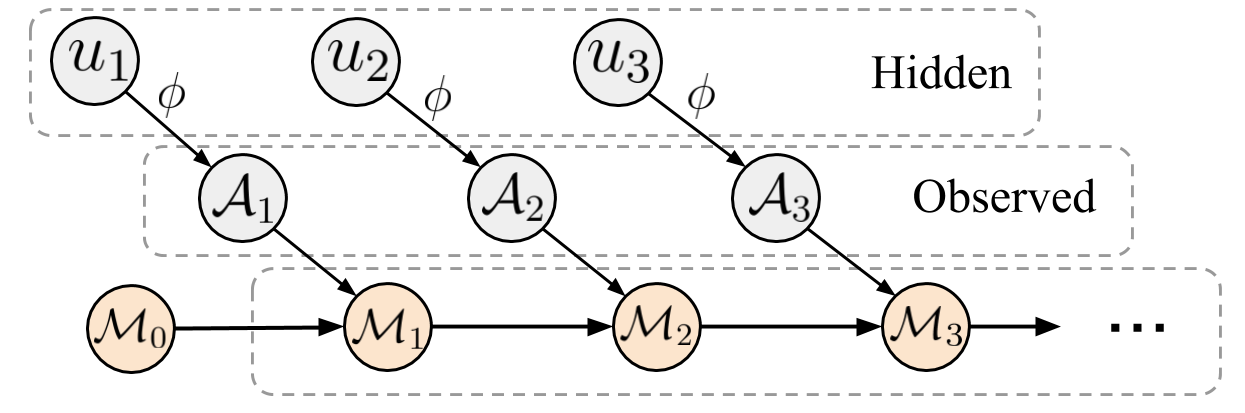}
		\caption{ Illustration of a \emph{lifelong MDP} where $\mathcal M_0$ is the base MDP. For every change $k$, $\mathcal M_K$ builds upon $\mathcal M_{k-1}$ by including the newly available set of actions $\mathcal A_k$. The internal structure in the space of actions is hidden and only a set of discrete actions is observed. 
		}
		\label{Fig:CL-MDP}
\end{figure}

        To formalize the lifelong MDP, we first introduce the necessary variables that govern when and how new actions are added.
        We denote the episode number using $\tau$.
        Let $I_\tau \in \{0, 1\}$ be a random variable that indicates whether a new set of actions are added or not at the start of episode $\tau$, and let frequency $\mathcal F: \mathbb N \rightarrow [0, 1]$ be the associated probability distribution over episode count, such that $\Pr(I_\tau=1) = \mathcal F(\tau)$. 
        %
        %
        %
        Let $U_\tau \in 2^{\mathcal E}$ be the random variable corresponding to the set of actions that is added before the start of episode $\tau$. 
        When $I_\tau=1$, we assume that $U_\tau\neq\emptyset$, and when $I_\tau=0$, we assume that $U_\tau = \emptyset$. 
        Let $\mathcal D_\tau$ be the distribution of $U_\tau$ when $I_\tau=1$, i.e., $U_\tau \sim \mathcal D_\tau$ if $I_\tau=1$. 
        We use $\mathcal D$ to denote the set $\{\mathcal D_\tau\}$ consisting of these distributions.
        Such a formulation using $I_\tau$ and $\mathcal D_\tau$ provides a fine control of when and how new actions can be incorporated.
        This allows modeling a large class of problems where both the distribution over the type of incorporated actions as well intervals between successive changes might be irregular. 
        Often we will not require the exact episode number $\tau$ but instead require $k$, which denotes the number of times the action set is changed.
        %
        %

        Since we do not assume that the agent knows the structure associated with the action, we instead provide actions to the agent as a set of discrete entities, $\mathcal A_k$. 
        To this end, we define $\phi$ to be a map relating the underlying structure of the new actions to the observed set of discrete actions $\mathcal A_k$ for all $k$, i.e., if the set of actions added is $u_k$, then 
        $\mathcal A_k=\{  \phi(e_i) | e_i \in u_k\}$. 
        Naturally, for most problems of interest, neither the underlying structure $\mathcal E$, nor the set of distributions $\mathcal D$, nor the frequency of updates $\mathcal F$, nor the relation $\phi$ is known---the agent only has access to the observed set of discrete actions.
        
        %
        %
        %

        We now define the \textit{lifelong Markov decision process} (L-MDP)  as $\mathscr{L} = (\mathcal M_0, \mathcal E, \mathcal D, \mathcal F)$, which extends a \textit{base} MDP $\mathcal M_0 = (\mathcal{S}, \mathcal{A},\mathcal{P},\mathcal{R}, \gamma, d_0)$. 
        %
        %
        %
        %
    	%
    	$\mathcal{S}$ is the set of all possible states that the agent can be in, called the state set.
    	$\mathcal A$ is the discrete set of actions available to the agent, and for $\mathcal M_0$ we define this set to be empty, i.e., $\mathcal A=\emptyset$. 
    	When the set of available actions changes and 
    	the agent observes a new set of discrete actions, $\mathcal A_k$, then $\mathcal M_{k-1}$ transitions to $\mathcal M_k$, such that  $\mathcal A$ in $\mathcal M_k$ is the set union of $\mathcal A$ in $\mathcal M_{k-1}$ and $\mathcal A_k$. 
    	Apart from the available actions, other aspects of the L-MDP remain the same throughout.
    	An illustration of the framework is provided in Figure \ref{Fig:CL-MDP}.
    	%
        %
        %
        %
        %
    	We use $S_t \in \mathcal S$, $A_t \in \mathcal A$, and $R_t \in \mathbb R$ as random variables for denoting the state, action and reward at time $t \in \{0,1,\dotsc\}$ within each episode.
    	%
        %
    	The first state, $S_0$, comes from an initial distribution, $d_0$, and the reward function $\mathcal R$ is defined to be only dependent on the state such that $\mathcal R(s)=\mathbf{E}[R_t|S_t=s]$ for all $s \in \mathcal S$. 
    	We assume that $R_t \in [-R_\text{max},R_\text{max}]$ for some finite $R_\text{max}$. 
    	The reward discounting parameter is given by $\gamma \in [0,1)$. 
    	$\mathcal{P}$ is the state transition function, such that for all $s,a,s',t$, the function $ \mathcal{P}(s,a,s')$ denotes the transition probability $ P(s'| s, e)$, where $a = \phi(e)$.\footnote{For notational ease, (a) we overload symbol $P$ for representing both probability mass and density; 
    	(b) we assume that the state set is finite, however, our primary results extend to MDPs with continuous states.
    	} 

    	In the most general case, new actions could be completely arbitrary and have no relation to the ones seen before.
    	In such cases, there is very little hope of lifelong learning by leveraging past experience.
    	To make the problem more feasible, we resort to a notion of \textit{smoothness} between actions.
    	Formally, we assume that transition probabilities in an L-MDP are $\rho-$Lipschitz in the structure of actions, i.e., $\exists \rho > 0$ s.t.,
    	\begin{equation} 
    	 \forall s, s', e_i, e_j \hspace{5pt} \lVert P(s'| s,e_i) - P(s'| s,e_j) \rVert_1 \leq \rho \lVert e_i - e_j\rVert_1. 
    	\label{eqn:lipschitz}
    	\end{equation}
    	%
    	%
    	%
    	%
    	%
    	For any given MDP $\mathcal{M}_k$ in $\mathscr L$, an agent's goal is to find a policy, $\pi_k$, that maximizes the expected sum of discounted future rewards.
    	For any policy $\pi_k$, the corresponding state value function is $v^{\pi_k}(s) = \mathbf{E}[\sum_{t=0}^{\infty}\gamma^t R_{t} |s, \pi_k]$. 
    	%
    	%
    	%
    	%

 \section{Blessing of Changing Action Sets}
 \label{sec:blessing}
Finding an optimal policy when the set of possible actions is large 
is 
difficult due to the curse of dimensionality.
In the L-MDP setting this problem might appear to be exacerbated, as an agent must additionally adapt to the changing levels of possible performance as new actions become available. 
This raises the natural question: \textit{as new actions become available, how much does the performance of an optimal policy change? }
%
If it fluctuates significantly, can a lifelong learning agent succeed by continuously adapting its policy, or is it better to learn from scratch with every change to the action set? 

To answer this question, 
consider an optimal policy, $\pi^*_k$, for MDP $\mathcal M_k$, i.e., an optimal policy when considering only policies that use actions that are available during the $k^\text{th}$ episode. 
%
%
We now 
quantify how sub-optimal $\pi^*_k$ is relative to the performance of a hypothetical policy, $\mu^*$, that acts optimally given access to all possible actions. 
\begin{thm}
\label{thm:1}
In an L-MDP, let $\epsilon_k$ denote the maximum distance in the underlying structure of the closest pair of available actions, i.e.,  
$\epsilon_k \coloneqq \underset{a_i \in \mathcal A}{\text{sup}} ~ \underset{a_j \in \mathcal A}{\text{inf}}  \lVert e_i - e_j \rVert_1$, 
then
\begin{align}
    v^{\mu^*}(s_0) -v^{\pi^*_k}(s_0)   &\leq \frac{\gamma  \rho \epsilon_k}{(1 - \gamma)^2}  R_{\text{max}}.  \label{eqn:thm1} 
\end{align}
\end{thm}
\begin{proof}
    See Appendix B. 
\end{proof}
With a bound on the maximum possible sub-optimality, Theorem \ref{thm:1} presents an important connection between achievable performances, the nature of underlying structure in the action space, and a property of available actions in any given $\mathcal M_k$.
Using this, we can make the following conclusion.
\begin{cor}
\label{cor:1}
Let $\mathcal Y\subseteq \mathcal E$ be the smallest closed set such that,  $P(U_k \subseteq 2^\mathcal Y)=1$. We refer to $\mathcal Y$ as the element-wise-support of $U_k$. If $\,\,\forall k$, the element-wise-support of $U_k$ in an L-MDP is $\mathcal E$, 
then as $k \rightarrow \infty$ the sub-optimality vanishes. That is,
$$\lim_{k \rightarrow \infty} v^{\mu^*}(s_0) -v^{\pi^*_k}(s_0)  \rightarrow 0. $$
\end{cor}
\begin{proof}
    See Appendix B. 
\end{proof}
Through Corollary \ref{cor:1}, we can now establish that the change in optimal performance will eventually converge to zero as new actions are repeatedly added. 
An intuitive way to observe this result would be to notice that every new action that becomes available indirectly provides more information about the underlying, unknown, structure of $\mathcal E$.
However, in the limit, as the size of the available action set increases, the information provided by each each new action vanishes and thus performance saturates.

Certainly, in practice, we can never have $k \rightarrow \infty$, but this result is still advantageous.
Even when the underlying structure $\mathcal E$, the set of distributions $\mathcal D$, the change frequency $\mathcal F$, and the mapping relation $\phi$ are all \textit{unknown}, it establishes the fact that the difference between the best performances in \textit{successive changes} will remain bounded and will not fluctuate arbitrarily.
%
%
This opens up new possibilities for developing algorithms that do 
not need to start from scratch after new actions are added, but rather can build upon their past experiences using updates to their existing policies that efficiently leverage estimates of the structure of $\mathcal E$ to adapt to new actions.  
%
%
%
\section{Learning with Changing Action Sets}
\label{sec:learning}
Theorem \ref{thm:1}  characterizes what \textit{can be} achieved in principle, however, it does not specify \textit{how} to achieve it---how to find $\pi_k^*$ efficiently.
Using any parameterized policy, $\pi$, which acts directly in the space of observed actions, suffers from one key practical drawback in the L-MDP setting.
That is, the parameterization is deeply coupled with the number of actions that are available. 
That is, not only is the meaning of each parameter coupled with the number of actions, but often the number of parameters that the policy has is dependent on the number of possible actions. 
This makes it unclear how the policy should be adapted when additional actions become available. 
A trivial solution would be to ignore the newly available actions and continue only using the previously available actions.
However, this is clearly myopic, and will prevent the agent from achieving the better long term returns that might be possible using the new actions.

To address this parameterization-problem, instead of having the policy, $\pi$, act directly in the observed action space, $\mathcal A$, we propose an approach wherein the agent reasons about the underlying structure of the problem in a way that makes its policy parameterization invariant to the number of actions that are available. 
To do so, we split the policy parameterization into two components.
The first component corresponds to the state conditional policy responsible for making the decisions, $\beta : \mathcal S \times \hat{\mathcal E} \rightarrow [0, 1]$, where $\hat {\mathcal E} \in \mathbb{R}^d$.
The second component corresponds to $\hat \phi : \hat{ \mathcal E} \times \mathcal A \rightarrow [0,1]$, an estimator of the relation $\phi$, which is used to map the output of $\beta$ to an action in the set of available actions.
That is, an $E_t \in \hat{\mathcal E}$ is sampled from $\beta(S_t, \cdot)$ and then $ \hat \phi(E_t)$ is used to obtain the action $A_t$. 
Together, $\beta$ and $\hat \phi$ form a complete policy, and $\hat {\mathcal E}$ corresponds to the inferred structure in action space.

One of the prime benefits of estimating $\phi$ with $\hat \phi$ 
is that it makes the parameterization of $\beta$ invariant to the cardinality of the action set---changing the number of available actions does not require changing the number of parameters of $\beta$.
%
%
Instead, only the parameterization of $\hat \phi$, the estimator of the underlying structure in action space, must 
be modified when new actions become available. 
We show next that the update to the parameters of $\hat \phi$ can be performed using \emph{supervised learning} methods that are independent of the reward signal and thus typically more efficient than RL methods. 
%
%
%

 %

%
%
%
 
While our proposed 
parameterization of the policy using both $\beta$ and $\hat \phi$ has the advantages described above, the performance of $\beta$ is now constrained by the quality of $\hat \phi$, as in the end $\hat \phi$ is responsible for selecting an action from $\mathcal A$.
Ideally we want $\hat \phi$ to be such that it lets $\beta$  be both: (a) invariant to the cardinality of the action set for practical reasons and (b) as expressive 
as a policy, $\pi$, explicitly parameterized for the currently available actions. 
Similar trade-offs have been considered in the context of learning optimal state-embeddings for representing sub-goals in hierarchical RL \citep{nachum2018near}. 
%
For our lifelong learning setting, we build upon their method to efficiently estimate $\hat \phi$ in a way that provides bounded sub-optimality. 
%
Specifically, we make use of an additional \textit{inverse dynamics} function, $\varphi$, that takes as input two states, $s$ and $s'$, and produces as output a prediction of which $e \in \mathcal E$ caused the transition from $s$ to $s'$. 
Since the agent does not know $\phi$, when it observes a transition from $s$ to $s'$ via action $a$, it does \emph{not} know which $e$ caused this transition. 
So, we cannot train $\varphi$ to make good predictions using the actual action, $e$, that caused the transition. 
Instead, we use $\hat \phi$ to transform the prediction of $\varphi$ from $e \in \mathcal E$ to $a \in \mathcal A$, and train both $\varphi$ and $\hat \phi$ so that this process accurately predicts which action, $a$, caused the transition from $s$ to $s'$. 
Moreover, rather than viewing $\varphi$ as a deterministic function mapping states $s$ and $s'$ to predictions $e$, we define $\varphi$ to be a \textit{distribution} over $\mathcal E$ given two states, $s$ and $s'$. 
%

For any given $\mathcal M_k$ in L-MDP $\mathscr L$, let $\beta_k$ and $\hat \phi_k$ denote the two components of the overall policy and let $\pi_k^{**}$ denote the best overall policy that can be represented using some fixed $\hat \phi_k$. 
The following theorem bounds the sub-optimality of $\pi_k^{**}$.
\begin{thm}
\label{thm:2} For an L-MDP $\mathcal M_k$,
If there exists a $\varphi : S \times S \times \hat{\mathcal E} \rightarrow [0, 1] $ and $\hat \phi_k : \hat {\mathcal E} \times \mathcal A \rightarrow [0, 1]$ such that 
\begin{align}
    \footnotesize
    \sup_{s \in \mathcal S, a \in \mathcal A} \text{KL}\Big(P(S_{t+1}|S_t=s, A_t=a) \Vert
    \\
     \,\,\,\,\,\,\,\,\,\,\,\,\,\,\,\,\,\,\, P(S_{t+1}|S_t=s, A_t=\hat A) \Big) &\leq \delta_k^2/2, \label{eqn:lemma1} 
\end{align}
where $\hat A 
\sim \hat \phi_k(\cdot|\hat E)$ and $\hat E \sim \varphi(\cdot | S_t, S_{t+1})$, then 
\begin{align*}
 v^{\mu^*}(s_0) -v^{\pi_k^{**}}(s_0)  &\leq \frac{\gamma \left( \rho \epsilon_k + \delta_k \right)}{(1 - \gamma)^2}  R_{\text{max}}.
\end{align*}
\end{thm}
%
\begin{proof}
    See Appendix B. 
\end{proof}

By quantifying the impact $\hat \phi$ has on the sub-optimality of achievable performance, Theorem \ref{thm:2} provides the necessary constraints for estimating $\hat \phi$. 
%
%
At a high level, Equation \eqref{eqn:lemma1} ensures $\hat \phi$ to be such that it can be used to generate an action corresponding to any $s$ to $s'$ transition.
%
%
This allows $\beta$ to leverage $\hat \phi$ and choose the required action that induces the state transition needed for maximizing performance.
Thereby, following \eqref{eqn:lemma1}, sub-optimality would be minimized if $\hat \phi$ and $\varphi$ are optimized to reduce the supremum of KL divergence over all $s$ and $a$. 
In practice, however, the agent does not have access to all possible states, rather it has access to a limited set of samples collected from interactions with the environment.
Therefore, instead of the supremum, we propose minimizing the average over all $s$ and $a$ from a set of observed transitions,
\begin{align}
    \!\!\! \mathcal L(\hat \phi, \varphi) \!\!&\coloneqq \!\! \sum_{s \in \mathcal S}\! \sum_{a \in \mathcal A_k} P(s, a) \textit{KL}\left(P(s'|s, a) \Vert P(s'|s, \hat a) \right) . \label{eqn:exp_kl}
\end{align}
Equation \eqref{eqn:exp_kl}  suggests that $\mathcal L(\hat \phi, \varphi)$ would be minimized when $\hat a$ equals $a$, but using \eqref{eqn:exp_kl}  directly in the current form is inefficient as it requires computing KL over all probable $s' \in \mathcal S$ for a given $s$ and $a$.
To make it practical, we make use of the following property.
\begin{prop}
\label{prop:lb}
For some constant C, 
$- \mathcal L(\hat \phi, \varphi)$ is lower bounded by
\begin{align}
     \sum_{s \in \mathcal S} \sum_{a \in \mathcal A_k}\sum_{s' \in \mathcal S} P(s, a, s') \Bigg (  \mathbf{E}\left[\log \hat \phi(\hat a|\hat e) \middle | \hat e \sim \varphi(\cdot|s,s')  \right]
     \\
     - \text{KL}\Big(\varphi(\hat e|s, s') \Big \Vert P(\hat e|s, s')  \Big) \Bigg) + C. ~~~~~\label{eqn:beta_vae}
\end{align}
\end{prop}
\begin{proof}
See Appendix C. 
\end{proof}
As minimizing $\mathcal L(\hat \phi, \varphi)$ is equivalent to maximizing $- \mathcal L(\hat \phi, \varphi)$, we consider maximizing the lower bound obtained from Property \ref{prop:lb}.
In this form, it is now practical to optimize \eqref{eqn:beta_vae} just by using the observed $(s, a, s')$ samples. 
As this form is similar to the objective for variational auto-encoder, inner expectation can be efficiently optimized using the reparameterization trick \citep{kingma2013auto}.
$P(\hat e | s, s')$ is the prior on $\hat e$, and we treat it as a hyper-parameter that allows the KL to be computed in closed form.

Importantly, note that this optimization procedure only requires individual transitions, $s,a,s'$, 
and is independent of the reward signal. 
Hence, at its core, it is a supervised learning procedure. 
This means that learning good parameters for $\hat \phi$ tends to require far fewer samples than optimizing $\beta$ (which is an RL problem). 
This is beneficial for our approach because $\hat \phi$, the component of the policy where new parameters need to be added when new actions become available, 
can be updated efficiently.
As both $\beta$ and $\varphi$ are invariant to action cardinality, they do not require new parameters when new actions become available. 
Additional implementation level details are available in Appendix F. 

\section{Algorithm}
\label{sec:algo}
 \begin{figure}[t]
        \centering
		\includegraphics[width=0.37\textwidth]{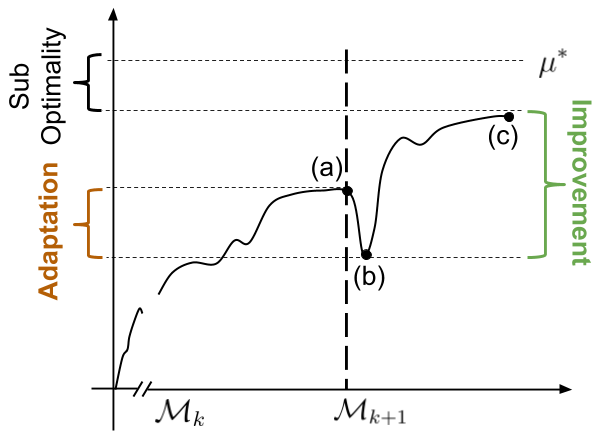}
		\caption{An illustration of a typical performance curve for a lifelong learning agent.
		The point $(a)$ corresponds to the performance of the current policy in $\mathcal M_k$.
		The point $(b)$ corresponds to the performance drop resulting as a consequence of adding new actions.
		We call the phase between (a) and (b) as the adaptation phase, which aims at minimizing this drop when adapting to new set of actions.
		The point $(c)$ corresponds to the improved performance in $\mathcal M_{k+1}$ by optimizing the policy to leverage the new set of available actions. 
		$\mu^*$ represents the best performance of the hypothetical policy which has access to the entire structure in the action space.
		}
		\label{Fig:LAICA}
\end{figure}

 When a new set of actions, $\mathcal A_{k+1}$, becomes available, the agent should leverage the existing knowledge and quickly adapt to the new action set.
 Therefore, during every change in $\mathcal M_k$, the ongoing best components of the policy, $\beta_{k-1}^*$ and $\phi_{k-1}^*$, in $\mathcal M_{k-1}$ are carried over, i.e., $\beta_k \coloneqq \beta_{k-1}^*$ and $\hat \phi_k \coloneqq \hat \phi_{k-1}^*$.
 For lifelong learning, the following property illustrates a way to organize the learning procedure so as to minimize the sub-optimality in each $\mathcal M_k$, for all $k$.  
 \begin{prop} (Lifelong Adaptation and Improvement)\label{prop:LAICA}
 In an L-MDP,  let $\Delta$ denote the difference of performance between $v^{\mu^*}$ and the best achievable using our policy parameterization, then the overall sub-optimality can be expressed as, 
  \begin{align}
    v^{\mu^*}(s_0) - v_{_{\mathcal M_1}}^{\beta_1 \hat \phi_1}(s_0) =&  \underbrace{\sum_{k=1}^{\infty}\left(v_{_{\mathcal M_k}}^{\beta_{k} \hat \phi_k^*}(s_0) - v_{_{\mathcal M_{k}}}^{\beta_{k}  \hat\phi_{k}}(s_0) \right)}_{\text{Adaptation}} \\
    &+ \underbrace{\sum_{k=1}^{\infty}\left(v_{_{\mathcal M_k}}^{\beta_k^*  \hat\phi_k^*}(s_0) - v_{_{\mathcal M_k}}^{\beta_k  \hat\phi_k^*}(s_0) \right)}_{\text{Policy Improvement}} + \Delta, 
 \label{eqn:adapt_improve}
  \end{align} 
  where $\mathcal M_k$ is used in the subscript to emphasize the respective MDP in $\mathscr L$. \textbf{
Proof:} See Appendix D.
 \end{prop}
%
%

Property \ref{prop:LAICA} illustrates a way to understand the impact of $\beta$ and $\hat \phi$ by splitting the learning process into an adaptation phase and a policy improvement phase.
These two iterative phases are the crux of our algorithm for solving an L-MDP $\mathscr L$.
Based on this principle, we call our algorithm LAICA: \textit{lifelong adaptation and improvement for changing actions}.
Due to space constraints, we now briefly discuss the LAICA algorithm; a detailed description with pseudocode is presented in Appendix E. 

Whenever new actions become available, adaptation is prone to cause a performance drop as the agent has no information about when to use the new actions, and so its initial uses of the new actions may be at inappropriate times. 
%
%
%
Following Property \ref{prop:lb}, we update $\hat \phi$ so as to efficiently infer the underlying structure and minimize this drop.
That is, for every $\mathcal M_k$, $ \hat\phi_k$ is first adapted to $ \hat \phi_k^*$ in the adaptation phase by adding more parameters for the new set of actions and then optimizing \eqref{eqn:beta_vae}.
After that, $ \hat \phi_k^*$ is fixed and $\beta_k$ is improved towards $\beta_k^*$ in the policy improvement phase, by updating the parameters of $\beta_k$ using the policy gradient theorem \citep{sutton2000policy}.
These two procedures are performed sequentially whenever $\mathcal M_{k-1}$ transitions to $\mathcal M_k$, for all $k$, in an L-MDP $\mathscr L$.
An illustration of the procedure is presented in Figure \ref{Fig:LAICA}.

    \begin{figure*}[t]
    		\centering
    		\includegraphics[width=0.8\textwidth]{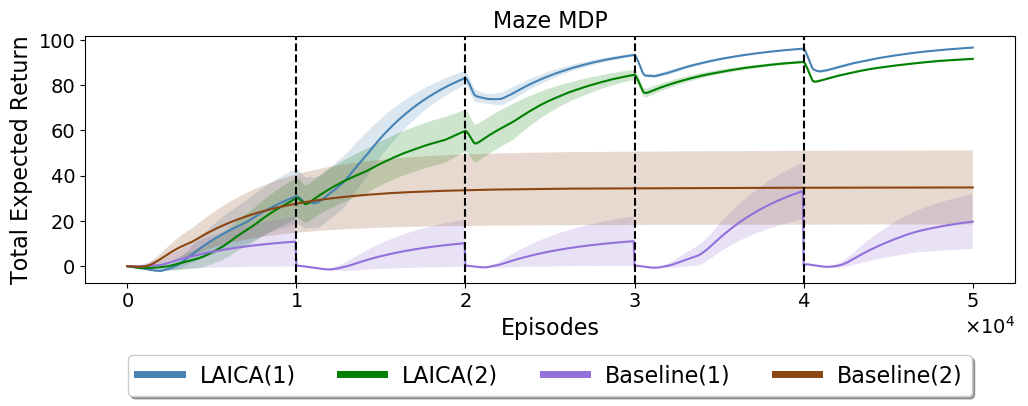}
    		\caption{Lifelong learning experiments with a changing set of actions in  the maze domain. 
    		The learning curves correspond to the running mean of the best performing setting for each of the algorithms.
    		The shaded regions correspond to standard error obtained using $10$ trials.
    		Vertical dotted bars indicate when the set of actions was changed.
    		%
    }
    		\label{Fig:CL-experiments}
	\end{figure*}
 
%


%
%
%
A step-by-step pseudo-code for the LAICA algorithm is available in Algorithm 1, Appendix E. 
The crux of the algorithm is based on the iterative adapt and improve procedure obtained from Property \ref{prop:LAICA}.

\section{Empirical Analysis}

            \begin{figure*}[t]
    		\centering
    		\includegraphics[width=0.8\textwidth]{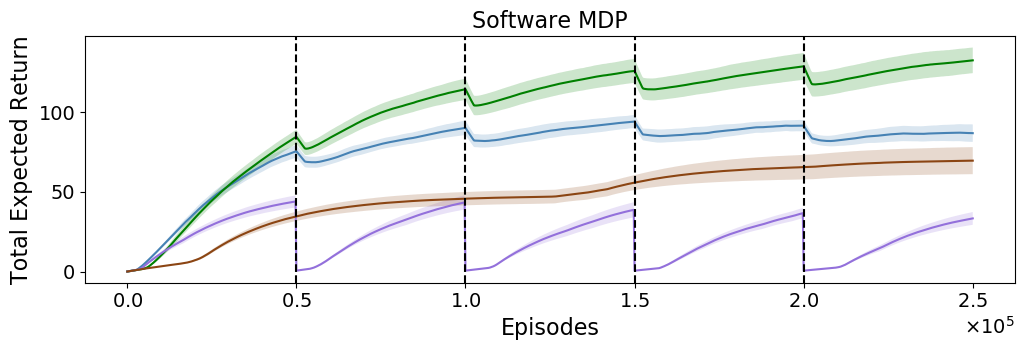}\\
    		\includegraphics[width=0.8\textwidth]{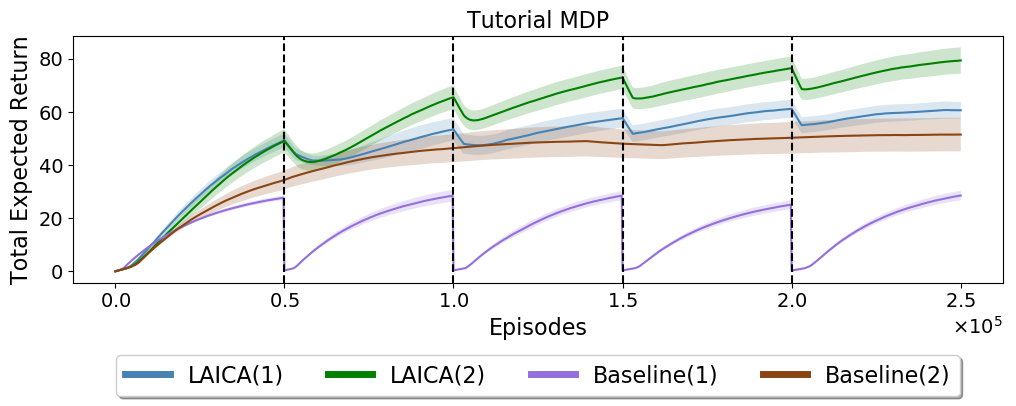}
    		\caption{Lifelong learning experiments with a changing set of actions in the recommender system domains. 
    		The learning curves correspond to the running mean of the best performing setting for each of the algorithms.  
    		The shaded regions correspond to standard error obtained using $10$ trials.
    		Vertical dotted bars indicate when the set of actions was changed.
    		}
    		\label{Fig:CL-experiments2}
		\end{figure*}
\label{sec:empirical}
        In this section, we aim to empirically compare the following methods, 
        \begin{itemize}
            \item Baseline(1): The policy is re-initialised  and the agent learns from scratch after every change.
            \item Baseline(2): New parameters corresponding to new actions are added/stacked to the existing policy (and previously learned parameters are carried forward as-is).
            \item LAICA(1): The proposed approach that leverages the structure in the action space. To act in continuous space of inferred structure, we use DPG \citep{silver2014deterministic} to optimize $\beta$.
            \item LAICA(2): A variant of LAICA which uses an actor-critic \citep{sutton2018reinforcement} to optimize $\beta$.
        \end{itemize}
        
    	To demonstrate the effectiveness of our proposed method(s) on lifelong learning problems, we consider a maze environment and two domains corresponding to real-world applications, all with a large set of changing actions. 
		For each of these domains, the total number of actions  were randomly split into five equal sets.
		Initially, the agent only had the actions available in the first set and after every change the next set of actions was made available additionally.
        In the following paragraphs we briefly outline the domains; full details are deferred to Appendix F. 
        
        \paragraph{Maze Domain.}
        As a proof-of-concept, we constructed a continuous-state maze environment where the state is comprised of the coordinates of the agent's location and its objective is to reach a fixed goal state. 
        The agent has a total of $256$ actions corresponding to displacements in different directions of different magnitudes.
        This domain provides a simple yet challenging testbed that requires solving a long horizon task using a large, changing action set, in presence of a single goal reward.
    	%
           
    	\paragraph{Case Study: Real-World Recommender Systems. } 
    	We consider the following two real-world applications of large-scale recommender systems that require decision making over multiple time steps and where the number of possible decisions varies over the lifetime of the system.
    	%
    	%
    	\begin{itemize}
    	    \item A web-based video-tutorial platform, that has a recommendation engine to suggest a series of tutorial videos. The aim is to meaningfully engage the users in a learning activity.
    	    %
    	    %
    	    In total, $1498$ tutorials were considered for recommendation.
            \item A professional multi-media editing software, where sequences of tools inside the software need to be recommended. The aim is to increase user productivity and assist users in quickly achieving their end goal. In total, $1843$ tools were considered for recommendation.  
    	\end{itemize}
    	
        For both of these applications, an existing log of user's click stream data was used to create an n-gram based MDP model for user behavior \citep{shani2005mdp}.
        Sequences of user interaction were aggregated to obtain over $29$ million clicks and $1.75$ billion user clicks for the tutorial recommendation and the tool recommendation task, respectively.    
        The MDP had continuous state-space, where each state consisted of the feature descriptors associated with each item (tutorial or tool) in the current n-gram. 
        %
        \subsection{Results.}
    	The plots in Figures \ref{Fig:CL-experiments} and \ref{Fig:CL-experiments2} present the evaluations on the domains considered.
    	The advantage of LAICA over Baseline(1) can be attributed to its policy parameterization.
    	The decision making component of the policy, $\beta$, being invariant to the action cardinality can be readily leveraged after every change without having to be re-initialized. 
    	This demonstrates that efficiently re-using past knowledge can improve data efficiency over the approach that learns from scratch every time.

	    Compared to Baseline(2), which also does not start from scratch and reuses existing policy, we notice that the variants of LAICA algorithm still perform favorably.
    	As evident from the plots in Figures \ref{Fig:CL-experiments} and \ref{Fig:CL-experiments2}, while Baseline(2) does a good job of preserving the existing policy, it fails to efficiently capture the benefit of new actions.
    	%
    	%
    	%
    	While the policy parameters in both LAICA and Baseline(2) are improved using policy gradients,  the superior performance of LAICA can be attributed to the adaptation procedure incorporated in LAICA which aims at efficiently inferring the underlying structure in the space of actions.
    	Overall LAICA(2) performs almost twice as well as both the baselines on all of the tasks considered. 
    	In the maze domain, even the best setting for Baseline(2) performed inconsistently.  
    	Due to the sparse reward nature of the task, which only had a big positive reward on reaching goal, even the best setting for Baseline(2) failed on certain trials, resulting in high variance.

    	Note that even before the first addition of the new set of actions, the proposed method performs better than the baselines.
    	This can be attributed to the fact that the proposed method efficiently leverages the underlying structure in the action set and thus learns faster.
    	Similar observations have been made previously \citep{dulac2015deep,he2015deep,bajpai2018transfer}.
    	
    	In terms of computational cost, the proposed method updates the inverse dynamics model and the underlying action structure only when there is a change in the action set (Algorithm 1).
    	Therefore, compared to the baselines, no extra computational cost is required for training at \textit{each} time-step.
    	However, the added computational cost does impact the  \textit{overall} learning process and is proportional to the number of times new actions are introduced.

\section{Discussion and Conclusion}
In this work we established first steps towards developing the lifelong MDP setup for dealing with action sets that change over time.
Our proposed approach then leveraged the structure in the action space such that an existing policy can be efficiently adapted to the new set of available actions.
Superior performances on both synthetic and large-scale real-world environments demonstrate the benefits of the proposed LAICA algorithm.
%

To the best of our knowledge, this is the first work to take a step towards addressing the problem of lifelong learning with a changing action set.
We hope that this brings more attention to such understudied problems in lifelong learning.
There are several important challenges open for future investigation.

In many real-world applications, often due to external factors, some actions are removed over time as well.
    For example, if a medicine becomes outdated, if a product is banned, etc.
    While our applications were devoid of this aspect, the proposed algorithm makes use of a policy parameterization that is invariant to the cardinality of the action set, and thus can support both addition and removal.
    Our proposed policy decomposition method can still be useful for selecting an available action whose impact on the state transition is most similar to the removed action.

    There can be several applications where new actions that are added over time have no relation to the previously observed actions. 
    For example, completely new product categories, tutorial videos on new topics, etc. 
    In such cases, it is unclear how to leverage past information efficiently.
    We do not expect our proposed method to work well in such settings.

%
%
 %
 %
 %

\section{Acknowledgement}

The research was supported by and partially conducted at Adobe
Research. 
We are also immensely grateful to the three anonymous reviewers who shared their insights and feedback. 

{\footnotesize
\bibliographystyle{aaai}
\bibliography{bibliography}
}

\clearpage
\appendix
\onecolumn
\section*{\centering Lifelong Learning with a Changing Action Set \\(Supplementary Material)}

\setcounter{lemma}{0}
\setcounter{thm}{0}
\setcounter{cor}{0}
\setcounter{prop}{1}

\section{A: Preliminary}
For the purpose of our results, we would require bounding the shift in the state distribution between two policies.
Techniques for doing so has been previously studied in literature \citep{kakade2002approximately,kearns2002near,pirotta2013safe,achiam2017constrained}.
Specifically, we cover this preliminary result based on the work by \citet{achiam2017constrained}.

The discounted state distribution, for all $s \in \mathcal S$, for a policy $\pi$ is given by,

\begin{align}
  d_\pi(s) = (1 - \gamma)\sum_t \gamma^t P(S_t=s|\pi). \label{eqn:dist}  
\end{align}
Let the shift in state distribution between any given policies $\pi_1$ and $\pi_2$ be denoted as $D(\pi_1, \pi_2)$, such that
\begin{align}
    D(\pi_1, \pi_2) &=    \int_{\mathcal{S}} \left | d_{\pi_1}(s) -  d_{\pi_2}(s) \right |  \mathrm{d}s  
    \\
    &=  \int_{\mathcal{S}} \left | (1 - \gamma)\sum_t \gamma^t P(S_t=s|\pi_1) - (1 - \gamma)\sum_t \gamma^t P(S_t=s|\pi_2)  \right | \mathrm{d}s . \label{eqn:3}
\end{align}
For any policy $\pi$, let $P^{\pi}$ denote the matrix corresponding to transition probabilities as a result of $\pi$.    
Then \eqref{eqn:3} can be re-written as,
\begin{align}
      D(\pi_1, \pi_2) &=  \left \lVert  (1 - \gamma)(1 - \gamma P^{\pi_1})^{-1}d_0 - (1 - \gamma)(1 - \gamma P^{\pi_2})^{-1}d_0   \right \rVert_1
    \\
    &=   \left \lVert (1 - \gamma) \middle ( (1 - \gamma P^{\pi_1})^{-1} - (1 - \gamma P^{\pi_2})^{-1} \middle)  d_0  \right \rVert_1. \label{eqn:qw}
\end{align}    
To simplify \eqref{eqn:qw}, let $G_1 =  (1 - \gamma P^{\pi_1})^{-1}$ and $G_2 = (1 - \gamma P^{\pi_2})^{-1} $.
\\
Then, 
$ G_1 - G_2 =  G_1(G_2^{-1} - G_1^{-1})G_2 $, and therefore \eqref{eqn:qw} can be written as,
\begin{align}    
     D(\pi_1, \pi_2) &= \left \Vert (1 - \gamma) \middle ( (1 - \gamma P^{\pi_1})^{-1} ((1 - \gamma P^{\pi_2})  - (1 - \gamma P^{\pi_1}))(1 - \gamma P^{\pi_2})^{-1} \middle)  d_0  \right \rVert_1
    \\
    &=  \left \lVert (1 - \gamma)\middle ( (1 - \gamma P^{\pi_1})^{-1} (\gamma P^{\pi_1} - \gamma P^{\pi_2})(1 - \gamma P^{\pi_2})^{-1} \middle)  d_0   \right \rVert_1
    \\
    &= \left \lVert \middle ( (1 - \gamma P^{\pi_1})^{-1}  (\gamma P^{\pi_1} - \gamma P^{\pi_2}) \middle) d_{\pi_2}  \right \rVert_1. \label{eqn:4}
\end{align}
Note that using matrix L1 norm,
\begin{align}
    \left \lVert(1 - \gamma P^{\pi_1})^{-1} \right \rVert_1 =  \left \lVert\sum_t (\gamma P^{\pi_1})^t \right \rVert_1 \leq  \sum_t \gamma^t \left \lVert P^{\pi_1^t} \right \rVert_1 = \sum_t \gamma^t \cdot \mathbf{1} = (1-\gamma)^{-1} \label{eqn:123}.
\end{align}
Combining \eqref{eqn:123} and \eqref{eqn:4},
\begin{align}
    D(\pi_1, \pi_2) &\leq \gamma (1- \gamma)^{-1} \left \lVert (P^{\pi_1} - P^{\pi_2}) d_{\pi_2}\right \rVert_1 . \label{eqn:5}
\end{align}

\section{B: Sub-Optimality}


\label{apx:sub-opt}
\begin{thm}
In an L-MDP, let $\epsilon_k$ denote the maximum distance in the underlying structure of the closest pair of available actions, i.e.,  
$\epsilon_k \coloneqq \underset{a_i \in \mathcal A}{\text{sup}} ~ \underset{a_j \in \mathcal A}{\text{inf}}  \lVert e_i - e_j \rVert_1$, 
then
\begin{align}
    v^{\mu^*}(s_0) -v^{\pi^*_k}(s_0)   &\leq \frac{\gamma  \rho \epsilon_k}{(1 - \gamma)^2}  R_{\text{max}}.  
\end{align}
\end{thm}

\begin{proof}
We begin by defining $\mu_k^*$ to be a policy where the actions of the policy $\mu^*$ is restricted to the actions available in $\mathcal M_k$.
That is, any action $e_i$ from $\mu^*$ is mapped to the closest $e_j$, where $a = \phi(e_j)$ is in the available action set.
Notice that the best policy, $\pi_k^*$, using the available set of actions is always better than or equal to $\mu_k^*$, i.e., $v^{\mu^*_k} \leq v^{\pi^*_k}$.
Therefore ,
\begin{align}
     v^{\mu^*}(s_0) -v^{\pi^*_k}(s_0)  &\leq \left \vert v^{\mu^*}(s_0) -v^{\mu^*_k}(s_0) \right |. \label{eqn:muk} 
\end{align}
On expanding the $v(s_0)$ corresponding for both the policies in \eqref{eqn:muk} using \eqref{eqn:dist},
\begin{align}
     v^{\mu^*}(s_0) -v^{\pi^*_k}(s_0) | &\leq \left | (1- \gamma)^{-1} \int_{\mathcal{S}} d_{\mu^*}(s)R(s)\mathrm{d}s - (1- \gamma)^{-1} \int_{\mathcal{S}} d_{\mu_k^*}(s)R(s)\mathrm{d}s \right |
    \\
    &= \left | (1- \gamma)^{-1} \int_{\mathcal{S}} \middle (d_{\mu^*}(s) -  d_{\mu_k^*}(s) \middle) R(s)\mathrm{d}s  \right |. \label{eqn:2}
\end{align}
We can then upper bound \eqref{eqn:2} by taking the maximum possible reward common,
\begin{align}
    v^{\mu^*}(s_0) -v^{\pi^*_k}(s_0)  &\leq  \left | (1- \gamma)^{-1} R_{\text{max}}\int_{\mathcal{S}} \middle (d_{\mu^*}(s) -  d_{\mu_k^*}(s) \middle) \mathrm{d}s  \right |
    \\
    &\leq  (1- \gamma)^{-1} R_{\text{max}}\int_{\mathcal{S}} \left | d_{\mu^*}(s) -  d_{\mu_k^*}(s)  \right | \mathrm{d}s 
    \\
    &=  (1- \gamma)^{-1} R_{\text{max}}D(\mu^*, \mu_k^*) 
    \\
     &\leq  \gamma(1- \gamma)^{-2} R_{\text{max}} \left \lVert (P^{\mu_k} - P^{\mu_k^*}) d_{\mu_k^*}\right \rVert_1 \label{eqn:wq}
\end{align}
For any action $\bar e$ taken by the policy $\mu^*$, let $\bar e_k$ denote the action for $\mu_k^*$ obtained by mapping $\bar e$ to the closest action in the available set,
then expanding \eqref{eqn:wq}, we get, 
\begin{align}
    v^{\mu^*}(s_0) -v^{\pi^*_k}(s_0) 
    &\leq  \gamma(1- \gamma)^{-2} R_{\text{max}}\left(\underset{s}{\text{sup}} \left  \lVert P^{\mu_k}(s) - P^{\mu_k^*}(s) \right \rVert_1 \right)
    \\
    &\leq  \gamma(1- \gamma)^{-2} R_{\text{max}}\left(\underset{s,s'}{\text{sup}} \left | \int_{\bar e} (P(s'|s, \bar e) - P(s'|s, \bar e_k)) \mu^*(\bar e|s) \mathrm{d}\bar e \right | \right)
    \\
    &\leq \gamma(1- \gamma)^{-2} R_{\text{max}} \left(\underset{s,s',\bar e}{\text{sup}} \middle | P(s'|s, \bar e) - P(s'|s, \bar e_k) \middle | \right). \label{eqn:6}
\end{align}
From the Lipschitz condition \eqref{eqn:lipschitz}, we know that $| P(s'|s, \bar e) - P(s'|s, \bar e_k)| \leq \rho \lVert\bar e - \bar e_k\rVert_1$.
As $\bar e_k$ corresponds to the closest available action for $\bar e$, the maximum distance for $\lvert \bar e - \bar e_k\rVert_1$ is bounded by $\epsilon_k$.
Combining \eqref{eqn:6} with these two observations, we get the desired result,
\begin{align}
     v^{\mu^*}(s_0) -v^{\pi^*_k}(s_0) &\leq  \frac{\gamma  \rho \epsilon_k}{(1 - \gamma)^2}  R_{\text{max}}.
\end{align}    
    
\end{proof}

\begin{cor}
Let $\mathcal Y\subseteq \mathcal E$ be the smallest closed set such that,  $P(U_k \subseteq 2^\mathcal Y)=1$. We refer to $\mathcal Y$ as the element-wise-support of $U_k$. If $\,\,\forall k$, the element-wise-support of $U_k$ in an L-MDP is $\mathcal E$, 
then as $k \rightarrow \infty$ the sub-optimality vanishes. That is,
$$\lim_{k \rightarrow \infty} v^{\mu^*}(s_0) -v^{\pi^*_k}(s_0)  \rightarrow 0. $$
\end{cor}

\begin{proof}
    Let $X_1, ..., X_n$ be independent identically distributed random vectors in $\mathcal E$. Let $X_i$ define a partition of $\mathcal E$ in $n$ sets $V_1, ... V_n$, such that $V_i$ contains all points in $\mathcal E$ whose nearest neighbor among $X_1, ..., X_n$ is $X_i$. Each such $V_i$ forms a \textit{Voronoi cell}.
    Now using the condition on full element-wise support, we know from the distribution free result  by \citet{devroye2017measure} that the diameter($V_i$) converges to $0$ at the rate $n^{-1/d}$ as $n \rightarrow \infty$ (Theorem 4, \citet{devroye2017measure}).
    As $\epsilon_k$ corresponds to the maximum distance between closest pair of points in $\mathcal E$, $\epsilon_k \leq \underset{i}{\sup} ~ \text{diameter}(V_i)$.
    Therefore, when $k \rightarrow \infty$ then $n \rightarrow \infty$; consequently $\epsilon_k \rightarrow 0$ and thus $ v^{\mu^*}(s_0) -v^{\pi^*_k}(s_0) \rightarrow 0$.
    %
    
\end{proof}

\begin{thm}
For an L-MDP $\mathcal M_k$,
If there exists a $\varphi : S \times S \times \hat{\mathcal E} \rightarrow [0, 1] $ and $\hat \phi_k : \hat {\mathcal E} \times \mathcal A \rightarrow [0, 1]$ such that 
\begin{align}
    \sup_{s \in \mathcal S, a \in \mathcal A} \text{KL}\left(P(S_{t+1}|S_t=s, A_t=a) \Vert P(S_{t+1}|S_t=s, A_t=\hat A) \right) \leq \delta_k^2/2, 
    \label{eqn:aszx} 
\end{align}
where $\hat A 
\sim \hat \phi_k(\cdot|\hat E)$ and $\hat E \sim \varphi(\cdot | S_t, S_{t+1})$, then 
\begin{align*}
 v^{\mu^*}(s_0) -v^{\pi_k^{**}}(s_0)  &\leq \frac{\gamma \left( \rho \epsilon_k + \delta_k \right)}{(1 - \gamma)^2}  R_{\text{max}}.
\end{align*}
\end{thm}

%

\begin{proof}
We begin by noting that,
\begin{align}
    v^{\mu^*}(s_0) -v^{\pi_k^{**}}(s_0) &= v^{\mu^*}(s_0) - v^{\pi_k^{*}}(s_0) + v^{\pi_k^{*}}(s_0) -v^{\pi_k^{**}}(s_0).  
\end{align}
Using Theorem \eqref{thm:1},
\begin{align}
  v^{\mu^*}(s_0) -v^{\pi_k^{**}}(s_0) &\leq  \frac{\gamma  \rho \epsilon_k}{(1 - \gamma)^2}  R_{\text{max}} + \left ( v^{\pi_k^{*}}(s_0) -v^{\pi_k^{**}}(s_0) \right). \label{eqn:total_sub}
\end{align}
Now we focus on bounding the last two terms in \eqref{eqn:total_sub}.
Following steps similar to \eqref{eqn:2} and \eqref{eqn:wq} it can bounded as,
    \begin{align}
         v^{\pi_k^*}(s_0) -v^{\pi_k^{**}}(s_0)    &\leq   \left | (1- \gamma)^{-1} R_{\text{max}}\int_{\mathcal{S}} \middle (d_{\pi_k^*}(s) -  d_{\pi_k^{**}}(s) \middle) \mathrm{d}s  \right |
         \\
            &\leq   (1- \gamma)^{-1} R_{\text{max}}\int_{\mathcal{S}} \left |  d_{\pi_k^*}(s) -  d_{\pi_k^{**}}(s)  \right | \mathrm{d}s 
        \\
        &=  (1- \gamma)^{-1} R_{\text{max}}D(\pi_k^*, \pi_k^{**})
         \\
        &= \gamma(1- \gamma)^{-2} R_{\text{max}} \left \lVert (P^{\pi_k^*} - P^{\pi_k^{**}}) d_{\pi_k^{**}}\right \rVert_1
        \\
        &\leq \gamma(1- \gamma)^{-2} R_{\text{max}} \left( \mathbb{E} \left[ 2TV \left ( P^{\pi_k^*}(s'|s) \Vert P^{\pi_k^{**}}(s'|s)\right) \middle| s \sim d_{\pi_k^{**}} \right ] \right),
    \end{align}
where $TV$ stands for total variation distance.
Using Pinsker's inequality,
\begin{align}
    v^{\pi_k^*}(s_0) -v^{\pi_k^{**}}(s_0) 
        &\leq \gamma(1- \gamma)^{-2} R_{\text{max}}   \left(\underset{s}{\text{sup}}  \sqrt{2 KL \left ( P^{\pi_k^*}(s'|s) \Vert P^{\pi_k^{**}}(s'|s) \right) }
          \right),
         \\
         &\leq \gamma(1- \gamma)^{-2} R_{\text{max}}   \left(\underset{s, a}{\text{sup}}  \sqrt{2 KL \left ( P(s'|s, a) \Vert P(s'|s, \hat a) \right) }
          \right),
\end{align}
where, $a \sim \pi_k^*$ and $\hat a \sim \pi_k^{**}$.
As condition \eqref{eqn:aszx} ensures that maximum KL divergence error between an actual $a$ and an action that can be induced through $\hat \phi_k$ for transitioning from $s$ to $s'$ is bounded by $\delta_k^2/2$, we get the desired result, 
\begin{align}
     v^{\pi_k^*}(s_0) -v^{\pi_k^{**}}(s_0) 
        &\leq \frac{\gamma \delta_k}{(1 - \gamma)^2}  R_{\text{max}}. \label{eqn:xcv}
\end{align}

Therefore taking the union bound on \eqref{eqn:total_sub} and \eqref{eqn:xcv}, we get the desired result
\begin{align*}
 v^{\mu^*}(s_0) -v^{\pi_k^{**}}(s_0)  &\leq \frac{\gamma \left( \rho \epsilon_k + \delta_k \right)}{(1 - \gamma)^2}  R_{\text{max}}.
\end{align*}

\end{proof}

\section{C: Lower Bound Objective For Adaptation}
\label{apx:lb}
\begin{align}
    \mathcal L(\hat \phi, \varphi) &=  \sum_{s \in \mathcal S} \sum_{a \in \mathcal A_k} P(s, a) \text{KL}\left(P(s'|s, a) \Vert P(s'|s, \hat a) \right) 
    \\
    &= - \sum_{s \in \mathcal S} \sum_{a \in \mathcal A_k} P(s, a) \sum_{s' \in \mathcal S} P(s'|s, a) \log P(s'|s, \hat a) + C_1   \label{eqn:qwe}
\end{align}
where $C_1$ is a constant corresponding to the entropy term in KL that is independent of $\hat a$. 
Continuing, we take the negative on both sides,
\begin{align}
    - \mathcal L(\hat \phi, \varphi) &= \sum_{s \in \mathcal S} \sum_{a \in \mathcal A_k}\sum_{s' \in \mathcal S} P(s, a, s') \log P(s'|s, \hat a) - C_1 
    \\
    &= \sum_{s \in \mathcal S} \sum_{a \in \mathcal A_k}\sum_{s' \in \mathcal S} P(s, a, s') \log \frac{P(s, \hat a, s')}{P(s, \hat a)} - C_1
    \\
    &= \sum_{s \in \mathcal S} \sum_{a \in \mathcal A_k}\sum_{s' \in \mathcal S} P(s, a, s') \log \frac{P(s, \hat a, s')}{\sum_{s' \in \mathcal S}P(s, \hat a, s')} - C_1
    \\
    &= \sum_{s \in \mathcal S} \sum_{a \in \mathcal A_k}\sum_{s' \in \mathcal S} P(s, a, s')\left [\log P(s, \hat a, s') - \log Z \right] - C_1
\end{align}
where $Z = \sum_{s' \in \mathcal S}P(s, \hat a, s')$ is the normalizing factor.
As $- \log Z$ is always positive, we obtain the following lower bound,
\begin{align}
    \\
    - \mathcal L(\hat \phi, \varphi) &\geq \sum_{s \in \mathcal S} \sum_{a \in \mathcal A_k}\sum_{s' \in \mathcal S} P(s, a, s') \log P(s, \hat a, s') - C_1
    \\
    &= \sum_{s \in \mathcal S} \sum_{a \in \mathcal A_k}\sum_{s' \in \mathcal S} P(s, a, s') \log P(\hat a|s, s')P(s, s') - C_1
    \\
    &= \sum_{s \in \mathcal S} \sum_{a \in \mathcal A_k}\sum_{s' \in \mathcal S} P(s, a, s') \log P(\hat a|s, s') + C_2 - C_1, \label{eqn:lbb}
\end{align}
where $C_2$ is another constant consisting of $\log P(s, s')$ and is independent of $\hat a$.

Now, let us focus on $P(\hat a|s,s')$, which represent the probability of the action $\hat a$ given the transition $s, s'$.
Notice that $\hat a$ is selected by $\hat \phi$ only using $\hat e$.
Therefore, given $\hat e$, probability of $\hat a$ is independent of everything else, 
\begin{align}
    \log P(\hat a|s,s') &=  \log \int P(\hat a| \hat e, s, s')P(\hat e | s, s') \mathrm{d}\hat e = \log \int P(\hat a|\hat e) P(\hat e | s, s') \mathrm{d}\hat e. \label{eqn:likelihood}
\end{align}
%
%
Let $Q(\hat e|s,s')$ be a parameterized distribution that encodes the context $(s, s')$ into the structure $\hat e$,
then, we can write \eqref{eqn:likelihood} as, 
\begin{align}
    \log P(\hat a|s,s') &=  \log \int \frac{Q(\hat e|s, s')}{Q(\hat e|s, s')} P(\hat a|\hat e) P(\hat e | s, s') \mathrm{d}\hat e
    \\
    &= \log \mathbf{E}\left[\frac{ P(\hat a|\hat e) P(\hat e|s, s')}{Q(\hat e|s, s')} \middle | Q(\hat e|s,s')  \right]
    \\
    &\geq \mathbf{E}\left[\log \frac{ P(\hat a|\hat e) P(\hat e| s, s')}{Q(\hat e|s, s')} \middle | Q(\hat e|s,s')  \right] \hspace{30pt} \text{(from Jensen's inequality)}
    \\
    &= \mathbf{E}\Big[\log P(\hat a|\hat e)  \Big | Q(\hat e|s,s')  \Big] + \mathbf{E}\left[\log \frac{ P(\hat e|s, s')}{Q(\hat e|s, s')} \middle | Q(\hat e|s,s')  \right]
    \\
    &= \mathbf{E}\Big[\log P(\hat a|\hat e) \Big | Q(\hat e|s,s')  \Big] - \text{KL}\Big(Q(\hat e|s, s') \Big \Vert P(\hat e | s, s')  \Big). \label{eqn:beta_vae_0}
\end{align}
Notice that $P(\hat a|e)$ and $Q(e|s, s')$ correspond to $\hat \phi$ and $\varphi$, respectively.
$P(\hat e | s, s')$ corresponds to the prior on $\hat e$.
Therefore, combining \eqref{eqn:lbb} and \eqref{eqn:beta_vae_0} we get,
\begin{align}
    - \mathcal L(\hat \phi, \varphi) \geq  \sum_{s \in \mathcal S} \sum_{a \in \mathcal A_k}\sum_{s' \in \mathcal S} P(s, a, s') \left (  \mathbf{E}\left[\log \hat \phi(\hat a|\hat e) \middle | \varphi(\hat e|s,s')  \right] - \text{KL}\Big(\varphi(\hat e|s, s') \Big \Vert P(\hat e|s, s')  \Big) \right) + C,
\end{align}
where $C$ denotes all the constants.

\section{D: Lifelong Adaptation and Improvement}
\label{apx:laica}

 \begin{prop} (Lifelong Adaptation and Improvement)
 In an L-MDP,  let $\Delta$ denote the difference of performance between $v^{\mu^*}$ and the best achievable using our policy parameterization, then the overall sub-optimality can be expressed as, 
  \begin{align}
    v^{\mu^*}(s_0) - v_{_{\mathcal M_1}}^{\beta_1 \hat \phi_1}(s_0) &=  \underbrace{\sum_{k=1}^{\infty}\left(v_{_{\mathcal M_k}}^{\beta_{k} \hat \phi_k^*}(s_0) - v_{_{\mathcal M_{k}}}^{\beta_{k}  \hat\phi_{k}}(s_0) \right)}_{\text{Adaptation}} + \underbrace{\sum_{k=1}^{\infty}\left(v_{_{\mathcal M_k}}^{\beta_k^*  \hat\phi_k^*}(s_0) - v_{_{\mathcal M_k}}^{\beta_k  \hat\phi_k^*}(s_0) \right)}_{\text{Policy Improvement}} + \Delta, 
  \end{align} 
  where $\mathcal M_k$ is used in the subscript to emphasize the respective L-MDP.
 \end{prop}

 \begin{proof}
 \begin{align}
     v^{\mu^*}(s_0) - v_{_{\mathcal M_1}}^{\beta_1 \hat\phi_1}(s_0) &= v^{\mu^*}(s_0) \pm v_{_{\mathcal M_1}}^{\beta_1 \hat\phi_1^*}(s_0) - v_{_{\mathcal M_1}}^{\beta_1 \hat\phi_1}(s_0)
     \\
     &= \left( v^{\mu^*}(s_0) - v_{_{\mathcal M_1}}^{\beta_1 \hat\phi_1^*}(s_0) \right) + \left( v_{_{\mathcal M_1}}^{\beta_1 \hat\phi_1^*}(s_0) - v_{_{\mathcal M_1}}^{\beta_1 \hat\phi_1}(s_0) \right)
     \\
     &= \left( v^{\mu^*}(s_0) \pm v_{_{\mathcal M_1}}^{\beta_1^* \hat\phi_1^*}(s_0) - v_{_{\mathcal M_1}}^{\beta_1 \hat\phi_1^*}(s_0) \right) + \left( v_{_{\mathcal M_1}}^{\beta_1 \hat\phi_1^*}(s_0) - v_{_{\mathcal M_1}}^{\beta_1 \hat\phi_1}(s_0) \right)
     \\
     &= \left( v^{\mu^*}(s_0) - v_{_{\mathcal M_1}}^{\beta_1^* \hat\phi_1^*}(s_0) \right) + \left( v_{_{\mathcal M_1}}^{\beta_1^* \hat\phi_1^*}(s_0) - v_{_{\mathcal M_1}}^{\beta_1 \hat\phi_1^*}(s_0) \right) + \left( v_{_{\mathcal M_1}}^{\beta_1 \hat\phi_1^*}(s_0) - v_{_{\mathcal M_1}}^{\beta_1 \hat\phi_1}(s_0) \right).
 \end{align}
 As $\beta_2 \coloneqq \beta_1^*$ and $\hat\phi_2 \coloneqq \hat\phi_1^*$ in $\mathcal M_2$, 
 \begin{align}
     v^{\mu^*}(s_0) - v_{_{\mathcal M_1}}^{\beta_1 \hat\phi_1}(s_0)
     &= \left( v^{\mu^*}(s_0) - v_{_{\mathcal M_2}}^{\beta_2 \hat\phi_2}(s_0) \right) + \left( v_{_{\mathcal M_1}}^{\beta_1^* \hat\phi_1^*}(s_0) - v_{_{\mathcal M_1}}^{\beta_1 \hat\phi_1^*}(s_0) \right) + \left( v_{_{\mathcal M_1}}^{\beta_1 \hat\phi_1^*}(s_0) - v_{_{\mathcal M_1}}^{\beta_1 \hat\phi_1}(s_0) \right).
 \end{align}
 Notice that we have expressed the sub-optimality in $\mathcal M_1$ as sub-optimality in $\mathcal M_2$, plus adaptation and a policy improvement terms in $\mathcal M_1$. 
 Expanding it one more time,
  \begin{align}
      v^{\mu^*}(s_0) - v_{_{\mathcal M_1}}^{\beta_1 \hat\phi_1}(s_0)
     &= \left( v^{\mu^*}(s_0) - v_{_{\mathcal M_3}}^{\beta_3 \hat\phi_3}(s_0) \right) + \\
     &\hspace{10pt} \sum_{k=1}^2 \left( v_{_{\mathcal M_k}}^{\beta_k^* \hat\phi_k^*}(s_0) - v_{_{\mathcal M_k}}^{\beta_k \hat\phi_k^*}(s_0) \right) + \sum_{k=1}^2\left( v_{_{\mathcal M_k}}^{\beta_k \hat\phi_k^*}(s_0) - v_{_{\mathcal M_k}}^{\beta_k \hat\phi_k}(s_0) \right).
  \end{align}
  It is now straightforward to observe the result by successively `unravelling' the sub-optimality in $\mathcal M_3$ in a similar fashion.
The final difference between $v^{\mu^*}$ and the best policy using our proposed parameterization is $\Delta$.
 \end{proof}

\section{E: Algorithm Details}
\label{apx:alg}
A step-by-step pseudo-code for the LAICA algorithm is available in Algorithm 1.
The crux of the algorithm is based on the iterative adapt and improve procedure obtained from Property \ref{prop:LAICA}.

We begin by initializing the parameters for $\beta_{0}^*, \hat \phi_{0}^*$ and $\varphi_{0}^*$.
In Lines $3$ to $5$, for every change in the set of available actions, instead of re-initializing from scratch, the previous best estimates for $\beta, \hat\phi$ and $\varphi$ are carried forward to build upon existing knowledge.
As $\beta$ and $\varphi$ are invariant to the cardinality of the available set of actions, no new parameters are required for them.
In Line $6$ we add new parameters in the function $\hat \phi$ to deal with the new set of available actions.

To minimize the adaptation drop, we make use of Property \ref{prop:lb}.
Let $\mathcal L^{\text{lb}}$ denote the lower bound for $\mathcal L$, such that,
$$ \mathcal L^{\text{lb}}(\hat \phi, \varphi) \coloneqq \mathbf{E}\left[\log \hat \phi(\hat A_t|\hat E_t) \middle | \varphi(\hat E_t|S_t,S_{t+1})  \right] - \lambda \text{KL}\left(\varphi(\hat E_t|S_t, S_{t+1}) \middle \Vert P(\hat E_t|S_t, S_{t+1})  \right).$$
Note that following the literature on variational auto-encoders, we have generalized \eqref{eqn:beta_vae} to use a Lagrangian $\lambda$ to weight the importance of KL divergence penalty \citep{higgins2017beta}.\footnote{Conventionally, the Lagrangian in VAE setting is denoted using $\beta$ \citep{higgins2017beta}. In our paper, to avoid symbol overload, we use $\lambda$ for it.}
When $\lambda = 1$, it degenrates to \eqref{eqn:beta_vae}.
We set the prior $P(\hat e|s, s')$ to be an isotropic normal distribution, which also allows KL to be computed in closed form  \citep{kingma2013auto}. 
From Line $7$ to $11$ in the Algorithm 1, random actions from the available set of actions are executed and their corresponding transitions are collected in a buffer.
Samples from this buffer are then used to maximize the lower bound objective $\mathcal L^{\text{lb}}$ and adapt the parameters of $\hat \phi$ and $\varphi$.
The optimized $\hat \phi^*$ is then kept fixed during policy improvement.

Lines $16$-$22$ correspond to the standard policy gradient approach for improving the performance of a policy.
In our case, the policy $\beta$ first outputs a vector $\hat e$ which gets mapped by $\hat \phi^*$ to an action.
The observed transition is then used to compute the policy gradient \citep{sutton2000policy} for updating the parameters of $\beta$ towards $\beta^*$.  
%
%
	%
If a critic is used for computing the policy gradients, then it is also subsequently updated by minimizing the TD error \citep{sutton2018reinforcement}.
This iterative process of adaption and policy improvement continues for every change in the action set size.
%
%
%
		\IncMargin{1em}
	\begin{algorithm2e}
		\label{Alg:1} 
		\caption{Lifelong Adaptation and Improvement for Changing Actions (LAICA)}
	    \textbf{Initialize} $\beta_{0}^*, \hat \phi_{0}^*, \varphi_{0}^*$.
	    \\
        \For {\text{change} $k = 1,2...$} 
        {
            
            $\beta_k \leftarrow \beta_{k-1}^*$  \tikzmark{top}
            \\
            $\varphi_k \leftarrow \varphi_{k-1}^*$ 
            \\
            $\hat \phi_k \leftarrow \hat \phi_{k-1}^*$ 
            \\
            Add parameters in $\hat \phi_k$ for new actions ~~~~~~~~\tikzmark{right} \tikzmark{bottom}
            \\
            
            \vspace{10pt}
    		Buffer $\mathbb{B} = \{\}$  \tikzmark{top2}
    		\\
           \For {$episode = 0,1,2...$}
            {
    			\For {$t = 0,1,2...$}
    			{
			    Execute random $a_t$ and observe $s_{t+1}$ \\
			    Add transition to $\mathbb{B}$ 
			    }
            }
           \For {$iteration = 0,1,2...$}
            {
                Sample batch $b  \sim \mathbb{B}$
                \\
                Update $\hat \phi_k$ and $\varphi_k$ by maximizing $\mathcal L^{\text{lb}}(\hat \phi_k, \varphi_k)$ for $b$ ~~~~~~~~~~\tikzmark{right2}
                
            }    \tikzmark{bottom2}
			 \\
            \vspace{8pt}
    		\tikzmark{top3}
            \For {$episode = 0,1,2...$}{
			\For {$t = 0,1,2...$} {
			    Sample $\hat e_t \sim \beta_k(\cdot|s_t) $ \\
			    Map $\hat e_t$ to an action $a_t$ using $ \hat \phi_k^*(e)$ \\
			    Execute $a_t$ and observe $s_{t+1}, r_{t}$ \\
			    Update $\beta_k$ using any policy gradient algorithm ~~~~~~~~~~~~\tikzmark{right3}\\
			    Update critic by minimizing TD error. \tikzmark{bottom3}
			}
		}
	} 
	\nonl 
	\AddNote{top}{bottom}{right}{Reuse past knowledge.}
	\AddNote{top2}{bottom2}{right2}{Adapt \\ $\hat \phi_k$ to $\hat \phi_k^*$  }
	\AddNote{top3}{bottom3}{right3}{Improve \\ $\beta_k$ to $\beta_k^*$}
	\end{algorithm2e}
	\DecMargin{1em}

\section{F: Empirical Analysis Details}
\label{apx:empirical}
\subsection{Domains}
\label{apx:domain}
To demonstrate the effectiveness of our proposed method(s) on lifelong learning problems, we consider a maze environment and two domains corresponding to real-world applications, all with large set of changing actions. 
		For each of these domains, the total number of actions  were randomly split into five mutually exclusive sets of equal sizes.
		Initially, the agent only had the actions available in the first set and after every change the next set of action was made available additionally.
		 For all our experiments, changes to the action set were made after equal intervals.
    		
        \paragraph{Maze. } As a proof-of-concept, we constructed a continuous-state maze environment where the state comprised of the coordinates of the agent's current location. 
        The agent has $8$ equally spaced actuators (each actuator moves the agent in the direction the actuator is pointing towards) around it, and it can choose whether each actuator should be on or off. 
        Therefore, the total number of possible actions is $2^8 = 256$. 
        The net outcome of an action is the vectorial summation of the displacements associated with the selected actuators.
        The agent is penalized at each time step to encourage it to reach the goal as quickly as possible.
        A goal reward is given when it reaches the goal position
        %
        %
        To make the problem more challenging, random noise was added to the action $10\%$ of the time and the maximum episode length was $150$ steps. 

    	\paragraph{Case Study: Real-world recommender systems. } 
    	We consider two real-world applications of recommender systems that require decision making over multiple time steps  and where the number of possible decisions can vary over the lifetime of the system.

        First, a web-based video-tutorial platform, which has a recommendation engine that suggests a series of tutorial videos on various software.
        On this tutorial platform, there is a large pool of available tutorial videos on several software and new videos are uploaded periodically.
        This requires the recommender system to keep adjusting to these changes constantly.
        The aim for the recommender system is to suggest tutorials so as to meaningfully engage the user on how to use these software and convert novice users into experts in their respective areas of interest.

        The second application is a professional multi-media editing software. 
	    Modern multimedia editing software often contain many tools that can be used to manipulate the media, and this wealth of options can be overwhelming for users.
	    Further, with every major update to the software, new tools are developed and incorporated into the software to enhance user experience. 
	    In this domain, an agent suggests which of the available tools the user may want to use next.
	    The objective is to increase user productivity and assist in achieving their end goal. 

        For both of these applications, an existing log of user's click stream data was used to create an n-gram based MDP model for user behavior \citep{shani2005mdp}.
        In the tutorial recommendation task, sequences of user interaction were aggregated to obtain over $29$ million clicks.   
        Similarly, sequential usage patterns of the tools in the multi-media editing software were collected to obtain a total of over $1.75$ billion user clicks.  
        Tutorials and tools that had less than $100$ clicks in total were discarded. 
        The remaining $1498$ tutorials and $1843$ tools for the web-based tutorial platform and the multi-media software, respectively, corresponds to the total number of actions.
        The MDP had continuous state-space, where each state consisted of the feature descriptors associated with each item (tutorial or tool) in the current n-gram. 
        Rewards were chosen based on a surrogate measure for difficulty level of tutorials and popularity of final outcomes of user interactions in the multi-media editing software, respectively. 
        %
        %
        
        %

\subsection{Implementation Details } 
	\label{apx:imp}

For the maze domain, single layer neural networks were used to parameterize both the actor and critic.
The learning rates for policy were searched over the range $[1e-2, 1e-4]$ and for critic it was searched over the range $[5e-2, 5e-4]$. 
State features were represented using the $3^\text{rd}$ order coupled Fourier basis \citep{konidaris2011value}. 
The discounting parameter $\gamma$ was set to $0.99$ and eligibility traces to $0.9$.
Since it was a toy domain, the output dimension of $\beta$ was kept fixed to $2$.
After every change in the action set, $500$ randomly drawn trajectories were used to update $\hat \phi$.
The value of $\lambda$ was searched over the range $[1e-2, 1e-4]$.

For the real-world environments, $2$ layer neural networks were used to parameterize both the actor and critic.
The learning rates for both were searched over the range $[1e-2, 1e-4]$.
Similar to prior works, the module for encoding state features was shared to reduce the number of parameters, and the learning rate for it was additionally searched over $[1e-2, 1e-4]$. 
The dimension of the neural network's hidden layer was searched over $\{64, 128, 256\}$.
The discounting parameter $\gamma$ was set to $0.9$. 
For actor-critic based results eligibility traces was set to $0.9$ and for DPG the target actor and policy update rate was fixed to its default setting of $0.001$. 
The output dimension of $\beta$ was searched over $\{16, 32, 64\}$.
After every change in the action set, samples from $2000$ randomly drawn trajectories were used to update $\hat \phi$.

For all the results of the LAICA, since the output of $\beta$ was defined over a continuous space, it was parameterized as the isotropic normal distribution.
The value for variance was kept fix for the Maze domain and was searched over $[0.5, 1.5]$.
For the real-world domains, the variance was parameterized and learned along with other parameters.
The function $\varphi$ was parameterized to concatenate the state features of both $s$ and $s'$ and use a single layer neural network to project to a space corresponding to the inferred structure in the actions.
The function $\hat \phi$ was linearly parameterized to compute a Boltzmann distribution over the available set of actions.
After every change in the action set, new rows were stacked in its weight matrix for generating scores for the new actions.
The learning rates for functions $\hat \phi$ and $\varphi$ were jointly searched over $[1e-2, 1e-4]$.
%


	%
	
	%
	As our proposed method decomposes the overall policy into two components, the resulting architecture resembles that of a one layer deeper neural network.
	Therefore, for the baselines, we ran the experiments with a hyper-parameter search for policies with additional depths $\{1, 2, 3\}$, each with different combinations of width $\{2, 16, 64\}$.
	The remaining architectural aspects and properties of the hyper-parameter search for the baselines were performed in the same way as mentioned above for our proposed method.
    For dealing with new actions, new rows were stacked in the weight matrix of the last layer of the policy in Baseline(2). 
    
In total, $200$ settings for each algorithm, for each domain, were uniformly sampled from the respective hyper-parameter ranges/sets mentioned.
Results from the best performing setting are reported in all the plots.
Each hyper-parameter setting was independently ran using $10$ different seeds to get the standard error of the performance.

\end{document}